\documentclass[twoside,11pt]{article} 
\usepackage{jmlr2e} 
\ShortHeadings{Theory of Curriculum Learning}{Weinshall \& Amir}

\usepackage{multicol}
\usepackage{microtype}
\usepackage{graphicx}
\usepackage{subfigure}
\usepackage{booktabs} % for professional tables
\usepackage{amssymb}
\usepackage{bbm}
\usepackage{amsmath}
\usepackage{amsmath}
\usepackage{enumerate}
\usepackage{color}

\DeclareMathOperator{\sgn}{sgn}
\DeclareMathOperator{\argmin}{argmin}
\usepackage{hyperref}

\newtheorem{theorem}{Theorem}
\numberwithin{theorem}{section}
\newtheorem{lemma}{Lemma}
\numberwithin{lemma}{section}

\newtheorem{corollary}{Corollary}
\numberwithin{corollary}{section}

\newcommand{\bx}{{\mathbf x}}
\newcommand{\bv}{{\mathbf v}}

\newcommand{\bX}{{\mathbf X}}
\newcommand{\bz}{{\mathbf z}}
\newcommand{\bw}{{\mathbf w}}
\newcommand{\ba}{{\mathbf a}}

\newcommand{\bs}{{\mathbf s}}

\newcommand{\iR}{\mathbb{R}}

\newcommand{\cH}{\mathcal{H}}
\newcommand{\iX}{\mathbb{X}}
\newcommand{\iE}{\mathbb{E}}

\newcommand{\iO}{\vec{\mathcal{O}}}
\newcommand{\OO}{{\mathcal{O}}}

\begin{document}

\title{Theory of Curriculum Learning, with Convex Loss Functions}

\author{\name Daphna Weinshall\thanks{corresponding author} \email daphna@mail.huji.ac.il \\
       \addr School of Computer Science and Engineering\\
       Hebrew University of Jerusalem\\
       Jerusalem 91904, Israel
       \AND
       \name Dan Amir \email dan.amir@mail.huji.ac.il \\
       \addr School of Computer Science and Engineering\\
       Hebrew University of Jerusalem\\
       Jerusalem 91904, Israel}

\editor{}
\maketitle

\begin{abstract}

Curriculum Learning - the idea of teaching by gradually exposing the learner to examples in a meaningful order, from easy to hard, has been investigated in the context of machine learning long ago. Although methods based on this concept have been empirically shown to improve performance of several learning algorithms, no theoretical analysis has been provided even for simple cases. To address this shortfall, we start by formulating an ideal definition of difficulty score - the loss of the optimal hypothesis at a given datapoint. We analyze the possible contribution of curriculum learning based on this score in two convex problems - linear regression, and binary classification by hinge loss minimization. We show that in both cases, the expected convergence rate decreases monotonically with the ideal difficulty score, in accordance with earlier empirical results. We also prove that when the ideal difficulty score is fixed, the convergence rate is monotonically increasing with respect to the loss of the current hypothesis at each point. We discuss how these results bring to term two apparently contradicting heuristics: curriculum learning on the one hand, and hard data mining on the other.

\end{abstract}

\medskip

\section{Introduction}
\label{sec:intro}

Many popular machine learning algorithms involve sampling of examples from a large labeled data set and gradually improving the model performance on those examples. In particular, any algorithm which employs Stochastic Gradient Descent (SGD) falls under this category. In the standard and most common form of SGD, examples are drawn uniformly from the data. This approach is well justified theoretically as it guarantees that the expected value of the gradient in each step equals the gradient of the empirical loss. 

Although this approach is both simple and theoretically sound, it differs dramatically from our intuition of how living organisms learn from examples. Both humans and animals usually benefit from seeing examples in a meaningful order defined by some curriculum. The efficacy of learning new concepts is usually improved or even only made possible when the learner is exposed to gradually more difficult examples or more complex concepts. The learner usually uses the easier examples to acquire capabilities which will allow the grasping of the more complex ones. This concept is well grounded in cognitive research, where it was investigated within both a behavioral approach \citep[e.g.][]{skinner1990behavior} and a computational approach \citep[e.g.][]{elman1993learning}.

The idea of incorporating the concept of curriculum learning into the framework of supervised machine learning has been introduced early on \citep[e.g.][]{ sanger1994neural}, while being identified as a key challenge for machine learning throughout \citep{mitchell1980need,mitchell2006discipline,wang2015basic}. Several formulations have been suggested both in the context of SGD \citep{bengio2009curriculum} and in the context of other iterative optimization algorithms \citep{kumar2010self}. Most empirical studies, involving non-convex problems for the most part, demonstrated beneficial effects of curriculum learning, including faster convergence rate and  better final performance. Even so, this approach has not been widely adopted by practitioners (but see \citep{oh2015action,schroff2015facenet}). Moreover, this idea has not been theoretically analyzed, and no guarantees have ever been obtained for its success even on simple learning problems.

One inherent limitation of current curriculum learning approaches is the absence of formal and general definition of the difficulty score of an example, and a method for generating a curriculum based on such definition automatically. In their empirical research, \citet{bengio2009curriculum} relied on manually crafted, domain-specific curriculum. This approach is rather limited since in many cases the manual definition of easier sub-tasks or subsets of examples is impossible to acquire, especially with large scale and complex data. Moreover, even when it is possible to manually design a curriculum, the scoring of difficulty based on human intuition may not correlate well with the difficulty of the example or sub-problem for a learning algorithm. 

The framework of Self Paced learning (SPL) \citep{kumar2010self} overcomes this limitation by focusing on the intrinsic information of the learner, namely the loss with respect to the learner's current hypothesis, in order to avoid the need to obtain a curriculum from an extrinsic source. In this approach, a new optimization problem is introduced where the training loss is minimized jointly with a regularizing term, which attaches greater significance to points that better fit the current learner's hypothesis (namely, incur lower loss). While SPL obviates the need for a predefined curriculum, new difficulties are introduced: the new optimization problem is more difficult to solve, while by relying only on the learner's training loss it is more susceptible to problems like over-fitting and instability of training. Moreover, the SPL heuristics seems to contradict other commonly used heuristics, which attach greater significance to points that \emph{do not} fit well with the current learner's hypothesis (namely, incur \emph{higher} loss). Examples include hard data mining \citep{shrivastava2016training} and boosting \cite{boosting98}.

In this paper, we tackle those challenges from a theoretical point of view by first presenting the general definition of Ideal Difficulty Score (IDS) - the loss of the optimal hypothesis with respect to the example. We then analyze the relation between this score and the contribution of an example to the convergence of SGD in the context of two convex optimization problems - linear regression and classification with hinge loss minimization. Our analysis shows that under some reasonable assumptions,  convergence rate is expected to decrease monotonically with the difficulty of the sampled examples. This analysis is consistent with empirical results as discussed above.

Another challenge involves the success of apparently contradictory methods, which are based on the idea that the more difficult examples should be given higher weight \citep{shrivastava2016training,boosting98}. We hypothesize that this apparent contradiction can be explained in part by some confusion in the literature with respect to how difficulty is measured. More specifically, we formally differentiate between the \emph{global} difficulty score as defined by the IDS, and the \emph{local} difficulty score as defined by the loss with respect to the current hypothesis. In agreement with the intuition underlying both approaches, we claim that ideally a learner should follow a curriculum based on  extrinsic (global) difficulty, while not "wasting time" on examples that are easy for the current (local) learning hypothesis. In accordance, we formally show, by analyzing again the problems of linear regression and hinge loss minimization, that when examples are drawn conditioned on some fixed global difficulty score, convergence rate monotonically increases with the local difficulty of the example.

In Practice, there is no easy way to define a curriculum based on the concept of \emph{Ideal Difficulty Score}, since the optimal hypothesis is not known to the learner. Nevertheless, many practical scenarios that employ machine learning involve a sequence of iterations of model improvement. In such scenarios, results from earlier iterations can be used to generate a curriculum for subsequent iterations. Another scenario involves transfer learning from a strong learner to a weaker learner. Thus, it has been shown by \citet{weinshall2018curriculum} that curriculum based on the stronger model's difficulty scores can be used to train the weak model faster, and to lead it to a better solution.

\paragraph{Related Work.} 
%Jiang et al. \citep{jiang2017mentornet} 
\citet{jiang2017mentornet} tackled the problem of automatic generation of curriculum by suggesting a general framework for joint training of two deep neural networks, where one network  which is referred to as MentorNet is trained to generate adaptive curriculum for the other network. In their work, they show both empirically and theoretically that the data-driven generation of curriculum by MentorNet can improve the learner robustness to noisy data. 

The apparent contradiction in empirical reports, showing the advantage of both curriculum learning and hard example mining, motivated %Chang et al. \citep{chang2017active} 
\citet{chang2017active} to suggest the active bias method. This method circumvents the problem of "easy vs. hard" by focusing on certainty instead of difficulty. In their approach, the training schedule is designed according to the model's prediction variance over the previous training steps, where distribution is biased in favor of examples with high prediction variance. 

Our approach differs from these two ideas  in that it tackles  the question of difficulty definition directly. In contrast, MentorNet and active bias can, in theory, learn to generate biases over the data distribution which do not necessarily reflect a difficulty based curriculum. Future work should examine whether any curriculum generated by these methods complies with the intuition derived from our theoretic results. Namely, a curriculum should rank the examples so that they are negatively correlated with some global difficulty score, and positively correlated with the local difficulty.

In the rhe rest of the paper, we first introduce some basic notations and definitions in Section~\ref{sec:defs}. In Sections~\ref{sec:regression},\ref{sec:hinge} we develop the theory and prove the main results for the convex problems of linear regression and hinge loss respectively.

\section{Notations and Definitions}
\label{sec:defs}

Let $\iX_i=\{(\bx_i,y_i)\}_{i=1}^n$  denote the training data, where $\bx\in\iR^d$ denotes the $i$-th data point and $y$ its corresponding label. Let $\mathcal{D}$ denote the data distribution from which a sequence of training examples $\bX_t = \{\bx_t,y_t\}_{t=1}^T$ is drawn. Let $\mathcal{H}$ denote a set of hypotheses $h_\bw$ defined by the parameters vector $\bw$. Let $L(\bX_t,h)$ denote the loss of hypothesis $h$ when given example $\bX_t$. Then, our standard SGD objective is to find $\bar h$ defined by $\bar\bw$, which minimizes the empirical loss
\begin{equation*}
    L_\mathcal{D}(h)=\mathbb{E}_{\bX_t\sim \mathcal{D}}(L(\bX_t, h))
\end{equation*}
This framework is usually referred to as Empirical Risk Minimization. We will assume henceforth that samples are drawn directly from $\mathcal{D}$, which will allow us to analyze the continuous relation between the examples' difficulty scores and the expected convergence rate.

We now formally define the \emph{Ideal Difficulty Score} or Global Difficulty as described earlier. Let $\bar h = \underset{h}{\arg\min} \;L_\mathcal{D}(h)$, then the \emph{Ideal Difficulty Score} of example $\bX$ is defined as
\begin{equation}
    \Psi(\bX)= g(L(\bX,\bar h))
    \label{eq:def-IDS}
\end{equation}
where $g()$ is a monotonic function. Similarly, we define the \emph{Local Difficulty score} of an example $\bX$ at iteration $t$ as
\begin{equation*}
    \Upsilon(\bX)= g(L(\bX,h_t))
\end{equation*}
where $h_t $ is the hypothesis at time (iteration) $t$. 

Given the sequence $\{\bX_t \}_{t=1}^T$, SGD computes a sequence of estimators $\{\bw_t\}_{t=1}^T$. Although in practice many variations of SGD are used which yield different optimization steps, we analyze here the basic form in which the update rule is: 
\begin{equation}
\label{eq:ranked-GD}
\bw_{t+1} = \bw_t - \eta \frac{\partial L(\bX_t,\bw)}{\partial \bw} \vert_{\bw=\bw_t}
\end{equation}
where $\eta$ is a hyper-parameter which controls the learning rate of the algorithm. Since our analysis focuses on the transient behavior of the learning algorithm and not on a sequence of update steps, we will not state explicitly the index $t$ when it is clear from context.

SGD is only guaranteed to converge to a local minimum of the loss function. We therefore limit our analysis to simple convex problems. We chose to focus on two popular models of two common problems - linear regression and classification by hinge loss minimization. Since SGD without curriculum converges to the global optimum in these convex problems, we focus on the analysis of the effect of introducing curriculum on the convergence rate of SGD.

\newpage
\section{Linear Regression}
\label{sec:regression}

In linear regression, the learner's goal is to predict a real value $y=h(\bx)$ for $\bx\in \mathbb{R}^d$, where $h \in \mathcal{H}$ is a linear function of $\bx$ and the loss is defined by the least squares function. Formally, using the notations above, the loss function can be written as follows:
\begin{equation}
    L(\bX, \bw) = (\ba\cdot\bx+b-y)^2 \doteq (\bx\cdot\bw -y)^2 
\label{eq:reg}
\end{equation}
where $\bw\doteq [\ba,b]^t\in\iR^{d+1}$ concatenates the linear separator and the bias term. With some abuse of notation, $\bx$ now denotes the vector $[\bx,1]^t\in\iR^{d+1}$. Let $\bs$ denotes the gradient step at time $t$. We obtain from (\ref{eq:ranked-GD}) and (\ref{eq:reg})
\begin{equation}
\label{eq:gradient_t}
    \bw_{t +1}=\bw_{t} - 2 \eta (\bx\cdot\bw -y)\bx =\bw_{t} - 2\bs
\end{equation}

\subsection{Convergence rate decreases with \emph{global difficulty}}

The main theorem in this sub-section states that the expected rate of convergence of SGD is monotonically \emph{decreasing} with the \emph{Difficulty Score} of the sample $\bX_t$. We prove it below for the gradient step as defined in (\ref{eq:ranked-GD}). If the size of the gradient step is fixed at $\eta$, a somewhat stronger theorem can be obtained where the constraint on the step size being small is not required.

Recall that $\bx,\bw\in\iR^{d+1}$. The convergence analysis in carried out in the parameter space $\bw\in\iR^{d+1}$, where parameter vector $\bw$ corresponds to a point, and data vector $\bx$ describes a hyperplane. In this space, let $\Omega_{\bx}$ denote the hyperplane on which the gradient step $\bs$ vanishes, i.e. $\bs=0$. It follows from (\ref{eq:gradient_t}) that this hyperplane is defined by $\bx\cdot\bw = y$, namely, $\bx$ defines its normal direction. This implies that the gradient step at time $t$ is perpendicular to $\Omega_{\bx}$ as illustrated in Fig.~\ref{fig:omega-plain}. Let $\bar\bz$ denote the projection of $\bar\bw$, the parameters of the optimal hypothesis, on $\Omega_{\bx}$. 

\begin{figure}[th!]
	\centering
	\includegraphics[width=0.7\textwidth]{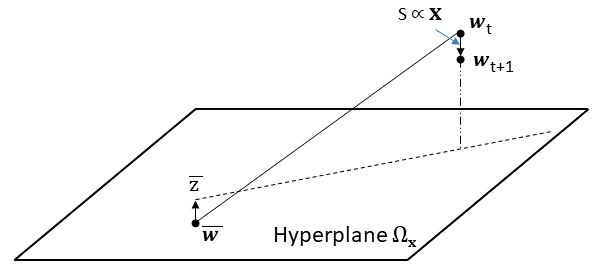}
    \caption{The geometry of the gradient step at time $t$. 
    \label{fig:omega-plain}}
\end{figure}

%We first write how the distance between $\bar \bw$ and $\Omega_{\bx}$ determines $\Psi$:
Because of the nature of the regression loss, which is based on the squared distance, we use $g(x)=\sqrt{x}$ in (\ref{eq:def-IDS}), giving us the following difficulty score $\Psi(\bX)=\sqrt{L(\bX,\bar\bw)}$. 
\begin{lemma}
\label{lemma:1}
Fix the training point $\bX$. The \emph{Difficulty Score} of $\bX$ is
$\Psi^2= r^2 \Vert \bar\bw-\bar\bz \Vert^2$.
\end{lemma}

\begin{proof}
\begin{equation}
\begin{split}
%\begin{align}
\label{eq:diff-score}
\Psi(\bX)^2 &=L(\bX,\bar\bw) = L(\bX,\bar\bz + (\bar\bw-\bar\bz))  = [ \bx\cdot\bar\bz + \bx\cdot (\bar\bw-\bar\bz) - y]^2 \\
&= [ \bx\cdot (\bar\bw-\bar\bz)]^2 = \Vert \bx \Vert^2 \Vert \bar\bw-\bar\bz \Vert^2 
\end{split}
\end{equation}
The first transition in the last line follows from $\bar\bz\in\Omega_{\bx} \implies \bx\cdot\bar\bz- y=0$. The second transition follows from the fact that both $\bx$ and $(\bar\bw-\bar\bz)$ are perpendicular to $\Omega_{\bx}$, and therefore parallel to each other.
\widowpenalty=10000
\end{proof}

Next, we embed the data points in the parameters space, representing each datapoint $\bx$ using a hyperspherical coordinate system $[r,\vartheta,\Phi]$, with pole (origin) fixed at $\bar\bw$ and polar axis (zenith direction) $\iO = \bar\bw-\bw_t$ (see Fig.~\ref{fig:obtuse}). $r$ denotes the vector's length, while $0\le\vartheta\le \pi$ denotes the polar angle with respect to $\iO$. Let $\Phi=[\varphi_1\ldots,\varphi_{d-1}]$ denote the remaining polar angles.

To illustrate, Fig.~\ref{fig:obtuse} shows a planar section of the parameter space, the $2D$ plane formed by the two intersecting lines $\iO$ and $\bar\bz-\bar\bw$. The gradient step $\bs$ points from $\bw_t$ towards $\Omega_{\bx}$. $\Omega_{\bx}$ is perpendicular to $\bx$, which is parallel to $\bar\bz-\bar\bw$ and to $\bs$, and therefore $\Omega_{\bx}$ is projected onto a line in this plane. We introduce the notation $\lambda=\Vert\bar\bw-\bw_t\Vert$. 

\begin{figure}[ht!]
	\centering
	\includegraphics[width=0.75\textwidth]{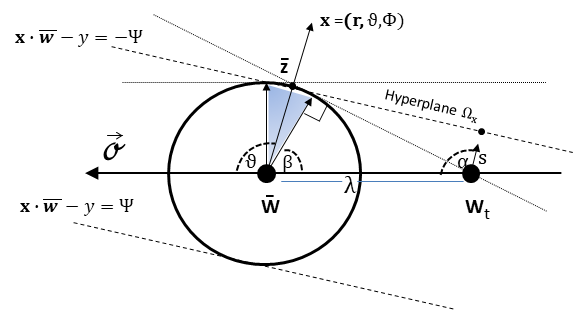}
    \caption{The $2D$ planar section defined by the vectors $\iO=\bar\bw-\bw_t$ and $\bar\bz-\bar\bw$. The circle centered on $\bar\bw$ has radius $\Vert\bar\bw-\bar\bz\Vert=\frac{{\Psi}}{\Vert\bx\Vert}$ from Lemma~\ref{lemma:1}. It traces the location of $\bar\bz$ for all the points $\bx$ with the same length $r$ and the same difficulty score $\Psi$. %The highlighted region corresponds to the range of angles $\vartheta$ for which $\vert \alpha\vert >\frac{\pi}{2}$.
    \label{fig:obtuse}}
\end{figure}

Let $\bs_{\OO}$ denote the projection of the gradient vector $\bs$ on the polar axis $\iO$, and let $\bs_{\perp}$ denote the perpendicular component. From (\ref{eq:gradient_t}) and the definition of $\Psi$
\begin{equation}
\label{eq:gradient_full}
\begin{split}
&\bs = -2\eta\bx (\bx\cdot \bw_t - y) = -2\eta\bx[\bx\cdot (\bw_t-\bar\bw)\pm{\Psi}] \\
&\bs_{\OO}=\bs\cdot\frac{\bar\bw-\bw_t}{\lambda}=2\frac{\eta}{\lambda}[r^2\lambda^2\cos^2\vartheta\mp{\Psi}r\lambda\cos\vartheta]
\end{split}
\end{equation}

Let $\bx=(r,\vartheta,\Phi)$. %Let $f_{{\cal D}(\iX)} = f(r,\vartheta,\Phi)f_Y(\vert y-\bx\cdot\bar\bw\vert)$ denote the density function of the data $\iX$. This choice assumes 
%the statistical independence of $r$ and $[\vartheta,\Phi]$. It also assumes that the density of the label $y$ only depends on the absolute error $\vert y-\bx\cdot\bar\bw\vert$.
%
The following analysis requires the conditional distribution of the data $\iX$ given difficulty score $\Psi$. We note that fixing the difficulty score determines the label to be one of the following two possible values for $y(\bx\vert_{\displaystyle{\Psi}})$: $y_1(\bx)=\bx\cdot\bar\bw+{\Psi}$, and $y_2(\bx)=\bx\cdot\bar\bw-{\Psi}$. We assume that both labels are equally likely, and therefore 
\begin{equation}
\label{eq:symm-assumption}
f_{{\cal D}(\iX)}(\bx,y_i(\bx)\vert_{\displaystyle{\Psi}}) = \frac{1}{2}f(\bx)=\frac{1}{2}f(r,\vartheta,\Phi)
\end{equation} 
This assumption implies a symmetrical data distribution ${{\cal D}(\iX)}$, where $f_{{\cal D}(\iX)}(\bx,\bx\cdot\bar\bw+u)=f_{{\cal D}(\iX)}(\bx,\bx\cdot\bar\bw-u)~\forall u$\footnote{In Appendix~\ref{app:bayesian} we show that in a Bayesian framework, this assumption can be replaced by $\iE_{f_v}[f_{{\cal D}(\iX)}(\bx,\bx\cdot\bar\bw+u)]=\iE_{f_v}[f_{{\cal D}(\iX)}(\bx,\bx\cdot\bar\bw-u)]~\forall u$, where expectation is taken with respect to some prior distribution ${f_v}$ over ${{\cal D}(\iX)}$.}.

%Let $-\pi\le\alpha\le\pi$ denote the angle between $\bs$, the gradient step induced by $\bx$ at point $\bw_t$, and $\iO$ - the ideal step direction at time $t$. 
Let $\Delta(\Psi)$ denote the expected convergence rate at time $t$, given fixed difficulty score $\Psi$. 
\begin{equation}
\label{eq:delta}
\Delta(\Psi) = \iE[\Vert \bw_{t} - \bar\bw \Vert^2 - \Vert \bw_{t+1} - \bar\bw \Vert^2 \vert_{\displaystyle{\Psi}}]
\end{equation}

\begin{lemma}
\label{lemma:2}
\begin{equation}
\label{eq:Delta}
\Delta(\Psi) = 2\lambda\iE[\bs_{\OO}\vert_{\displaystyle{\Psi}}]-\iE[\bs^2\vert_{\displaystyle{\Psi}}]
\end{equation}
\end{lemma}

\begin{proof}
From (\ref{eq:delta})
\begin{align*}
\iE[\Delta] &= (-\lambda)^2-\iE[(-\lambda+\bs_{\OO})^2+\bs_{\perp}^2] = \lambda^2-(\lambda^2 -2\lambda\iE[\bs_{\OO}]+\iE[\bs_{\OO}^2]) - \iE[\bs_{\perp}^2]  \\
&= 2\lambda\iE[\bs_{\OO}]-\iE[\bs^2]
\end{align*}
\widowpenalty=10000
\end{proof}
%Proof is given in Appendix~\ref{app:proofs}. 

Using this Lemma, it follows from (\ref{eq:gradient_full}) and (\ref{eq:Delta}) that\footnote{The short-hand notation $\iE[(\pm{\Psi})]$ implies that the operand of $\iE[]$ should appear with both signs, each with the matching probability of one of the 2 conditional labels $\{y_i(\bx)\}_{i=1}^2$.}
\begin{align}
\label{eq:e-delta}
%\begin{split}
 \frac{1}{4}\Delta(\Psi) &= \eta\iE[r^2\lambda^2\cos^2\vartheta]-\eta^2\iE[r^4\lambda^2\cos^2\vartheta] \nonumber -\eta^2\Psi^2\iE[r^2] \\
-& \eta\iE[(\pm{\Psi})r\lambda\cos\vartheta] - 2\eta^2\iE[(\pm{\Psi})r^3\lambda\cos\vartheta] 
%\end{split}
\end{align}
Using the symmetry assumption (\ref{eq:symm-assumption}), we can show that
\begin{equation*}
\iE[(\pm{\Psi})r\lambda\cos\vartheta] =\iE[(\pm{\Psi})r^3\lambda\cos\vartheta] = 0
\end{equation*}
from which it follows that
\begin{equation}
\label{cor:3}
\frac{1}{4}\Delta(\Psi) = \eta\iE[r^2\lambda^2\cos^2\vartheta]-\eta^2\iE[r^4\lambda^2\cos^2\vartheta] -\eta^2\Psi^2\iE[r^2] 
\end{equation}

We can now state the main theorem of this section.
\begin{theorem}
\label{theorem:1}
At time $t$ the expected convergence rate for training point $\bx$ is monotonically decreasing with the Difficulty Score $\Psi(\bX)$. If the step size coefficient is sufficiently small so that $\eta\le\frac{\iE[r^2\cos^2\vartheta]}{\iE[r^4\cos^2\vartheta]}$, it is likewise monotonically increasing with the distance $\lambda$ between the current estimate of the hypothesis $\bw_t$ and the optimal hypothesis $\bar\bw$.
\end{theorem}

\begin{proof}
From (\ref{cor:3})
\begin{equation*}
\frac{\partial\Delta(\Psi)}{\partial\Psi} = -8\eta^2\iE[r^2]\Psi \le 0
\end{equation*}
which proves the first statement. In addition,
\begin{equation*}
\frac{\partial\Delta(\Psi)}{\partial\lambda} = 8\eta\lambda \left ( \iE[r^2\cos^2\vartheta]-\eta\iE[r^4\cos^2\vartheta] \right )
\end{equation*}
If $\eta\le\frac{\iE[r^2\cos^2\vartheta]}{\iE[r^4\cos^2\vartheta]}$ then $\frac{\partial\Delta(\Psi)}{\partial\lambda}\ge 0$, and the second statement follows.
\widowpenalty=10000
\end{proof}

\begin{corollary}
\label{cor:1}
Although $\iE[\Delta(\Psi)]$ may be negative, $\bw_t$ always converges faster to $\bar\bw$ when the training points are sampled from easier examples with smaller $\Psi$. 
\end{corollary}

\begin{corollary}
\label{cor:2}
If the step size coefficient $\eta$ is small enough so that $\eta\le\frac{\iE[r^2\cos^2\vartheta]}{\iE[r^4\cos^2\vartheta]}$, we should expect faster convergence at the beginning of curriculum-based SGD. 
\end{corollary}

We note, outside the scope of the present discussion, that the predictions of these two corollaries have been observed in simulations with deep CNN network, where the loss function is far from being convex, see \citet{weinshall2018curriculum}.

\subsection{Convergence rate increases with \emph{local difficulty}}
\label{sec:regression-local}

The main theorem in this sub-section states that for a fixed global difficulty score $\Psi$, when the gradient step is small enough, convergence is monotonically \emph{increasing} with the \emph{local difficulty}, or the loss of the point with respect to the current hypothesis. \emph{This is not true in general.} The second theorem in this section shows that when the difficulty score is not fixed, there exist hypotheses $\bw\in\cH$ for which the convergence rate is decreasing with the local difficulty.

Let $\Upsilon^2=L(\bX,\bw_t)$ denote the loss of $\bX$ with respect to the current hypothesis $\bw_t$. Define the angle $\beta\in [0,\frac{\pi}{2})$ as follows (see Fig.~\ref{fig:obtuse})
\begin{equation}
\label{eq:beta}
\beta=\beta(r,\Psi,\lambda) = \arccos (\min(\frac{ {\Psi}}{\lambda r},1))
\end{equation}

\begin{lemma}
\label{lemma:4}
The relation between $\Upsilon, \Psi, r, \vartheta$ can be written separately in 4 regions as follows (see Fig.~\ref{fig:obtuse}):
\begin{eqnarray*}
A1& ~~0\le\vartheta\le\pi-\beta,~y=\bx\cdot\bar\bw+{\Psi}&\implies y=\bx\cdot\bw_t+\Upsilon, \lambda r \cos\vartheta=\bx\cdot(\bar\bw-\bw_t)=-\Psi+\Upsilon \\
A2& ~~\pi-\beta\le\vartheta\le\pi,~y=\bx\cdot\bar\bw+{\Psi}&\implies y=\bx\cdot\bw_t-{\Upsilon}, \lambda r \cos\vartheta=-{\Psi}-{\Upsilon} \\
A3& ~~ 0\le\vartheta\le\beta,~y=\bx\cdot\bar\bw-{\Psi}&\implies y=\bx\cdot\bw_t+{\Upsilon}, \lambda r \cos\vartheta={\Psi}+{\Upsilon}\\
A4& ~~ \beta\le\vartheta\le\pi,~y=\bx\cdot\bar\bw-{\Psi}&\implies y=\bx\cdot\bw_t-{\Upsilon},\lambda r \cos\vartheta={\Psi}-{\Upsilon}
\end{eqnarray*}
\end{lemma}

\begin{proof}
We keep in mind that $\forall \bx$ and $\Psi$, there are 2 possible labels $y$ whose probability is equal from assumption (\ref{eq:symm-assumption}). Recall that $\bar\bz$ denotes the projection of $\bar\bw$ on $\Omega_{\bx}$. In the planar section shown in Fig.~\ref{fig:obtuse},

\begin{description}
\item
$\bar\bz$ lies in the upper half space $\iff$ $y=\bx\cdot\bar\bw+{\Psi}$
\item
$\bar\bz$ lies in the lower half space  $\iff$ $y=\bx\cdot\bar\bw-{\Psi}$
\end{description}
This follows from 3 observations: $\bar\bx$ lies in the upper half space by the definition of the polar coordinate system, $\bx\cdot\bar\bw-y=\pm\Psi$, and
\begin{equation*}
0 = \bx\cdot\bar\bz - y= \bx\cdot (\bar\bz-\bar\bw)  +\bx\cdot\bar\bw-y
\end{equation*}

Next, let $\bz_t$ denote the projection of $\bw_t$ on $\Omega_{\bx}$. Then
\begin{equation*}
0 = \bx\cdot\bz_t - y= \bx\cdot (\bz_t-\bw_t)  +\bx\cdot\bw_t-y
\end{equation*}
When $\bar\bz$ lies in the upper half space, the following can be verified geometrically from Fig.~\ref{fig:obtuse}:
\begin{eqnarray*}
%\begin{split}
0\le\vartheta\le\pi-\beta &\implies~ \bx\cdot (\bz_t-\bw_t)\ge 0 &~implies~ y=\bx\cdot\bw_t+{\Upsilon} \\
\pi-\beta\le\vartheta\le\pi&\implies~\bx\cdot (\bz_t-\bw_t)\le 0 &\implies~ y=\bx\cdot\bw_t-{\Upsilon}
%\end{split}
\end{eqnarray*}

%\vspace{-3mm}
\widowpenalty=10000
\end{proof}

%Proof is given in Appendix~\ref{app:proofs}. 

Next we analyze how the convergence rate at $\bx$ changes with $\Upsilon$. Let $\Delta(\Psi,\Upsilon)$ denote the expected convergence rate at time $t$, given fixed global difficulty $\Psi$ and local difficulty  $\Upsilon$. From (\ref{cor:3}), $\Delta(\Psi,\Upsilon)=4\eta\iE[r^2\lambda^2\cos^2\vartheta\vert_{\displaystyle{\Upsilon}}]+O(\eta^2)$.

It is easier to analyze $\Delta(\Psi,\Upsilon)$ in a Cartesian coordinates system, rather than polar. We focus again on the $2D$ plane defined by the vectors $\iO=\bar\bw-\bw_t$ and $\bar\bz-\bar\bw$ (see Fig.~\ref{fig:obtuse}); here we define $u = r \cos\vartheta, ~v= r \sin\vartheta$. The 4 cases listed in Lemma~\ref{lemma:4} can be readily transformed to this coordinate system as follows $\{0\le\vartheta\le\beta\}\Leftrightarrow \{\lambda u\ge\Psi\}$, $\{\beta\le\vartheta\le\pi-\beta\}\Leftrightarrow \{-\Psi\le \lambda u\le\Psi\}$, and $\{\pi-\beta\le\vartheta\le\pi\}\Leftrightarrow \{\lambda u\le -\Psi\}$:

\begin{eqnarray*}
A1~~~&  \lambda u\ge -\Psi&\implies~~\lambda u=-{\Psi}+{\Upsilon}\\
A2~~~&  \lambda u\le -\Psi&\implies~~\lambda u=-{\Psi}-{\Upsilon}\\
A3~~~&  \lambda u\ge\Psi&\implies~~\lambda u={\Psi}+{\Upsilon}\\
A4~~~&  \lambda u\le\Psi&\implies~~\lambda u={\Psi}-{\Upsilon}
\end{eqnarray*}

Define
\begin{equation*}
\nabla = \frac{f(\frac{\Psi + \Upsilon}{\lambda}) -f(\frac{\Psi - \Upsilon}{\lambda})-f(\frac{-\Psi + \Upsilon}{\lambda})+f(\frac{-\Psi - \Upsilon}{\lambda})}{f(\frac{\Psi + \Upsilon}{\lambda}) +f(\frac{\Psi - \Upsilon}{\lambda})+f(\frac{-\Psi + \Upsilon}{\lambda})+f(\frac{-\Psi - \Upsilon}{\lambda})}
\end{equation*}
Clearly $-1 \leq \nabla \leq 1$. 

\begin{theorem}
\label{theorem:2}
Assume that the gradient step size is small enough so that we can neglect second order terms $O(\eta^2)$, and that $\frac{\partial\nabla}{\partial\Upsilon} \geq \frac{\Psi}{\Upsilon} - \frac{\Upsilon}{\Psi}~\forall\Upsilon$. Fix the difficulty score at $\Psi$. At time $t$ the expected convergence rate is monotonically increasing with the local difficulty $\Upsilon(\bx)$.
\end{theorem}

\begin{proof}
In the coordinate system defined above $\Delta(\Psi,\Upsilon)=4\eta\iE[\lambda^2 u^2\vert_{\displaystyle{\Upsilon}}]+O(\eta^2)$. We compute $\Delta(\Psi,\Upsilon)$ separately in each region, marginalizing out $v$ based on the following
\begin{equation*}
\int  \int_0^{\infty} \lambda^2 u^2 v^{d-1}f(u,v) dv du  = \int \lambda^2 u^2 f(u)du
\end{equation*}
where $f(u)$ denotes the marginal distribution of $u$. 

Let $u_i$ denote the value of $u$ corresponding to score $\Upsilon$ in each region A1-A4, and $\frac{1}{2}f(u_i)$ its density. $\Delta(\Psi,\Upsilon)$ takes on 4 discrete values, one in each region, and its expected value is therefore $\Delta(\Psi,\Upsilon)=4\eta\sum_{i=1}^4 \lambda^2 u_i^2 \frac{f(u_i)}{\sum_{i=1}^4 f(u_i)}$. It can readily be shown that
\begin{equation}
\begin{split}
\frac{1}{4\eta}\Delta&(\Psi,\Upsilon)= %\sum_{i\in\lbrace3,4\rbrace} u^2 \frac{2f(u)}{\sum_{i\in\lbrace3,4\rbrace} 2f(u)} \\
%&=\frac{(\Psi + \Upsilon)^2 f(\Psi + \Upsilon)+(\Psi - \Upsilon)^2 f(\Psi - \Upsilon)}{ f(\Psi + \Upsilon)+f(\Psi - \Upsilon)} \\ &=
 \Psi^2 + \Upsilon^2 +2\Psi\Upsilon\nabla
\end{split}
\end{equation}
and subsequently 
\begin{equation}
\begin{split}
\frac{1}{4\eta}\frac{\partial\Delta(\Psi,\Upsilon)}{\partial\Upsilon} &= 2\Upsilon+2\Psi\Upsilon
~\frac{\partial\nabla}{\partial\Upsilon}+2\Psi~\nabla \\
&\geq 2\Upsilon + 2\Psi\Upsilon
~\frac{\partial\nabla}{\partial\Upsilon}-2\Psi
\end{split}
\end{equation}

From the assumption that $\frac{\partial\nabla}{\partial\Upsilon} \geq \frac{\Psi}{\Upsilon} - \frac{\Upsilon}{\Psi}~\forall\Upsilon$, it follows that
\begin{equation*}
\frac{1}{8\eta}\frac{\partial\Delta(\Psi,\Upsilon)}{\partial\Upsilon} \geq \Upsilon +\Psi\Upsilon~\frac{\Psi-\Upsilon}{\Psi\Upsilon}-\Psi=0
\end{equation*}
\widowpenalty=10000
\end{proof}

\begin{corollary}
\label{corol:3}
For any $c\in \mathbb{R}^+$, if $\nabla$ is $(c-\frac{1}{c})$-Lipschitz then $\frac{\partial\Delta(\Psi,\Upsilon)}{
\partial\Upsilon} \geq 0$ for any $\Upsilon\geq c~\Psi$.
\end{corollary}
%Dan end here

\begin{corollary}
\label{corol:4}
If  ${\cal D}(\iX\vert_{\displaystyle{\Psi}})=k(\Psi)$ over a compact region and $\eta$ small enough, then $\frac{\partial\Delta(\Psi,\Upsilon)}{\partial\Upsilon} \geq 0$ for all $\Upsilon$ excluding the boundaries of the compact region. If in addition $\Upsilon>\Psi$, then $\frac{\partial\Delta(\Psi,\Upsilon)}{\partial\Upsilon} \geq 0$ almost surely.
\end{corollary}

\begin{theorem}
\label{theorem:3}
%Assume that the gradient step size is small enough so that we can neglect second order terms $O(\eta^2)$. Unless ${\cal D}$ is rotationally symmetric, 
Assume  ${\cal D}(\iX)$ is continuous and $\bar\bw$ is realizable. Then there are always hypotheses $\bw\in{\cH}$ for which the expected convergence rate under ${\cal D}(\iX)$ is monotonically decreasing with the local difficulty $\Upsilon(\bX)$.
\end{theorem}

\begin{proof}
We shift to a hyperspherical coordinate system in $\iR^{d+1}$ similar as before, but now the pole (origin) is fixed at $\bw_t$. For the gradient step $\bs$, it can be shown that:
\begin{equation}
%\label{eq:gradient_full}
\begin{split}
\bs &= -\sgn{(\bx\cdot\bw_t-y)}2\eta\bx \Upsilon   \\
\bs_{\OO}&=\bs\cdot\frac{\bar\bw-\bw_t}{\lambda}=\pm \frac{2\eta}{\lambda}r\lambda\cos\vartheta~{\Upsilon}
\end{split}
\end{equation}

Let $\Delta(\Upsilon)$ denote the expected convergence rate at time $t$, given fixed $\Upsilon$. From Lemma~\ref{lemma:2}
\begin{equation*}
\begin{split}
\Delta(\Upsilon) = 2\eta\Upsilon\bigg (&\iE[r\cos\vartheta\vert_{\displaystyle{\bx\cdot\bw_t-y=-\Upsilon}}] - 
\\ &\iE[r\cos\vartheta\vert_{\displaystyle{\bx\cdot\bw_t-y=\Upsilon}}]\bigg ) - \iE[(2\eta r\Upsilon)^2] \\
&\hspace{-15mm}\doteq 2\eta\Upsilon Q(r,\vartheta,\bw_t) - 4\eta^2\Upsilon^2\iE[r^2]
\end{split}
\end{equation*}

If $\bw=\bar\bw$, then $Q(r,\vartheta,\bw)=0$ from the symmetry of ${\cal D}(\iX)$ with respect to $\Psi$. From the continuity of ${\cal D}(\iX)$, there exists $\delta>0$ such that if $\Vert \bw-\bar\bw\Vert_2<\delta$, then $\Vert Q(r,\vartheta,\bw)-Q(r,\vartheta,\bar\bw)\Vert_2<\eta\Upsilon\iE[r^2]$, which implies that $\Delta(\Upsilon)<-2\eta^2\Upsilon^2\iE[r^2]<0$.
\widowpenalty=10000
\end{proof}

\section{Classification with the Hinge Loss}
\label{sec:hinge}

We now analyze hinge loss optimization in the context of binary classification. As in (\ref{eq:reg}), we adopt the notation where $\bx$ denotes the vector $[\bx,1]^t\in\iR^{d+1}$. The hypothesis $\bw\in \mathbb{R}^{d+1}$ defines a linear separator which includes a bias term, and the predicted class for  example $\bx$ is ${y}=sign(\bx\cdot \bw)$. The hinge loss function is defined as:
\begin{equation}
L(\bX, \bw) = \max(1-(\bx\cdot\bw) y , 0)
\label{eq:hinge}
\end{equation}

Since in (\ref{eq:hinge}) the margin is fixed at 1, it is desirable (and customarily done) to force a constraint on the length of the parameters vector $\Vert\bw\Vert$. Without loss of generality we use the constraint $\Vert\bw\Vert^2=1$ (see Appendix~\ref{app:general_normalization} for the relaxation of this constraint), which leads to the following optimization problem with Lagrange multiplier $\lambda$:
\begin{equation}
\bar\bw = \mathop{\argmin}_\bw \left [ \max(1-(\bx\cdot\bw) y , 0) + \lambda \Vert \bw \Vert ^2 \right ]
\label{eq:soft-SVM}
\end{equation}
Note that (\ref{eq:soft-SVM}) defines the soft-margin SVM classifier. 

When using GD, instead of minimizing the argument of (\ref{eq:soft-SVM}), one can minimize (\ref{eq:hinge}) directly in each step and subsequently project the solution onto the feasible set (aka \emph{projected gradient descent}). This is the procedure we analyze here, with the following update rule (similar to (\ref{eq:gradient_t}))
\begin{equation}
    \bw_{t +1}=\bw_{t } +\eta \hspace{0.2mm}\bs, ~~~
    \bs = \left.
  \begin{cases}
   \bx y &  (\bx\cdot\bw) y  \leq 1 \\
    0 &  elsewhere
  \end{cases}
  \right\}
\label{eq:step}
\end{equation}
Projection $\bw_{t+1}=\frac{\bw_{t+1}}{\Vert\bw_{t+1}\Vert}$ follows this gradient step. 

Given the normalization constraint on the parameters vector $\bw$, a suitable metric for comparing two such vectors is the cosine similarity between them (or their normalized inner product), in preference over the Euclidean distance between the vectors. We therefore define the expected convergence rate for a given \emph{Difficulty Score} $\Psi$ as
\begin{equation*}
    \Delta (\Psi) = \iE \left[ \frac{\bw_{t  + 1}\cdot \bar{\bw}}{\Vert \bw_{t  + 1}\Vert \hspace{0.1mm}\Vert \bar{\bw} \Vert } - \frac{\bw_{t }\cdot \bar{\bw}}{\Vert \bw_{t }\Vert \hspace{0.1mm}\Vert \bar{\bw} \Vert } \; \bigg{\vert}_{\displaystyle{\Psi}} \right]
\end{equation*}
Note that by definition 
$\Vert \bar{\bw} \Vert = \Vert \bw_{t } \Vert = 1$. 
\begin{figure}[th!]
	\centering
	\includegraphics[width=0.6\textwidth]{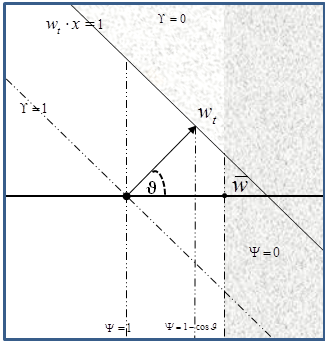}
    \caption{The geometry of the gradient step at time $t$ (see text). 
    \label{fig:hinge-plain}}
\end{figure}

Because the hinge loss is piecewise linear, we use the identity function $g(x)=x$ in the definition of the \emph{Difficulty Scores}, so that $\Psi(\bX)=L(\bX,\bar\bw)$ and $\Upsilon(\bX)=L(\bX,\bw_t)$. 

In the following analysis we use a fixed Cartesian coordinate system where the first coordinate axis is defined by $\bar{\bw}$, and the plane defined by the first and second axes is the subspace spanned by $\bar{\bw}$ and $\bw_{t }$ (see Fig.~\ref{fig:hinge-plain}). We assume w.l.o.g that $y=1$ (similar analysis applies to the symmetrical case of $y=-1$). By definition, in this coordinate system we have
\begin{equation}
\label{eq:coors}
    \bar{\bw} = [1,0,\ldots]^t
     ,~~~
\bw_{t } = [\cos\vartheta, \sin\vartheta \ldots]^t
\end{equation}
where $0\le\vartheta\le\pi$ denotes the angle between $\bar{\bw}$ and $\bw_{t }$. 

It follows that all points with \emph{Difficulty Score} $\Psi>0$ lie on a hyperplane defined by $\bx\cdot\bar\bw=1-\Psi$, where from (\ref{eq:coors})
\begin{equation}
\label{eq:x-vec}
     \bx\vert_{\displaystyle{\Psi}} = [1-\Psi,x_2,\cdots,x_{d+1}]^t
\end{equation}
The expected convergence rate can now be written as follows
\begin{equation}
\label{eq:delta_reduced}
    \Delta (\Psi) = \iE \left[  
    \frac{\cos{\vartheta} + \eta(1-\Psi)}{\Vert \bw_{t  + 1} \Vert} - \cos{\vartheta} \; \bigg{\vert}_{\displaystyle{\Psi}}
    \right]
\end{equation}

\subsection{Convergence rate decreases with \emph{global difficulty}}
\label{sec:hinge-global}

The main theorem in this section states that when minimizing the hinge loss, the expected convergence rate decreases with the global difficulty score $\Psi$. %In order to prove it, we need to make the following assumptions:
% \begin{defn}[Assumptions]
% $ $\newline
% \vspace{-7mm}
% \label{def:assumptions}
% \begin{enumerate}
% \item The learning rate $\eta$ is small enough so that $ \mathcal{O}(\eta ^2)$  terms can be neglected.
% %\item Under the given Cartesian coordinate system, $x_1 \perp x_2$.
% \item $\cos \vartheta \geq 0$.
% \end{enumerate}
%\end{defn}
%The first assumption has also been needed when discussing the linear regression in Section~\ref{sec:regression}. The second assumption implies that the current hypothesis $\bw_t$ is positively correlated with the optimal hypothesis $\bar\bw$. This is expected to usually hold early in the training procedure, due to the fact that in a high dimensional space two randomly picked vectors are expected to be almost orthogonal to each other, and therefore only a small step towards the optimal hypothesis is needed in order to satisfy the condition.

Before stating the first lemma, we note that from (\ref{eq:step})-(\ref{eq:x-vec}) 
\begin{equation}
\label{eq:hinge-tplus-norm}
    \Vert \bw_{t  + 1} \Vert = \sqrt{1+2\eta[(1-\Psi)\cos\vartheta+x_{2}\sin\vartheta]+\eta^{2}\Vert \bx \Vert ^ 2} 
\end{equation}
and
\begin{equation*}
    \bar{\bw}\cdot \bw_{t+1} = \cos \vartheta + \eta (1 - \Psi)
\end{equation*}

\begin{lemma}
\label{lemma:gradlim}
Let $\bX=[\bx,y]$ denote an example with Difficulty Score $\Psi > 0$, then
\begin{equation}
    \bx\cdot\bw_t < 1 \iff 
    x_2 < \frac{\Psi-1}{\tan\vartheta}+\frac{1}{\sin\vartheta}  \doteq  \mathcal{B}(\Psi) 
\label{eq:bound}
\end{equation}
\end{lemma}
\begin{proof}
From (\ref{eq:coors})-(\ref{eq:x-vec}) it follows that
\begin{equation*}
    \bx\cdot \bw_t =  \left(1 - \Psi \right)\cos{\vartheta} + x_2\sin{\vartheta} 
\end{equation*}
and therefore 
\begin{equation*}
    \bx\cdot \bw_t < 1 ~\iff~ \frac{\cos{\vartheta}}{\sin{\vartheta}}\left(  1 - \Psi \right) + x_2 <  \frac{1}{\sin{\vartheta}}
\end{equation*}
\widowpenalty=10000
\end{proof}

Lemma~\ref{lemma:gradlim} defines the range of $x_2$ for which $\Upsilon>0$, namely, the local \emph{Difficulty Score} is positive (see Fig.~\ref{fig:hinge-plain}), while the global \emph{Difficulty Score} is fixed at $\Psi$. This can be used to compute $\Delta (\Psi)$ from (\ref{eq:delta_reduced}) and obtain
\begin{lemma}
\label{lemma:integral}
Assume $\eta$ is small enough, then 
\begin{equation*}
\begin{split}
    \Delta (\Psi ) &= \int_{-\infty}^{\mathcal{B}(\Psi)}
    \eta [
    (1-\Psi)\sin^2\vartheta -x_2 \sin{\vartheta} \cos{\vartheta}]\\ &\hspace{1.25in}\cdot f(x_2) dx_2 +O(\eta^2)
\end{split}
\end{equation*}
% \begin{equation*}
%     \mathcal{P} = \sqrt{
%     1 + 2\eta\cdot (\cos{\vartheta}(1-\Psi) +\sin{\vartheta}\cdot x_2)} 
% \end{equation*}
where $f(x_2)$ denotes the marginal distribution of $\bx$ over the second axis, and ${\mathcal{B}(\Psi)}$ is defined in (\ref{eq:bound}).
\end{lemma}
\begin{proof}
Recall that the first coordinate of $\bx$ with \emph{Difficulty Score} fixed at $\Psi>0$ is constant at $x_1=1-\Psi$. We compute $\Delta (\Psi)$ using (\ref{eq:delta_reduced})  and Lemma~\ref{lemma:gradlim}:
\begin{equation}
\label{eq:fullintegral}
\begin{split}
    \Delta (\Psi) &= \int_{-\infty}^{\mathcal{B}(\Psi)}\ldots\int \mathcal{I}
    ~ f(x \vert_{\displaystyle{\Psi}}) dx_{d+1} \ldots dx_2 \\
    \mathcal{I} &= \frac{\cos{\vartheta} + (1-\Psi)\cdot \eta}{ \Vert \bw_{t+1} \Vert} -\cos{\vartheta}  
\end{split}
\end{equation}
where $\Vert \bw_{t+1} \Vert$ is defined in (\ref{eq:hinge-tplus-norm}).

Under the assumption that $\eta$ is small enough, we approximate the integrand $\mathcal{I}$ in (\ref{eq:fullintegral}) using the first terms of its Taylor expansion at $\eta = 0$, which yields
\begin{equation*}
\begin{split}
    \mathcal{I} &\approx \eta \Bigl[ 
    (1-\Psi) - \frac{2\bigl( (1-\Psi)\cos{\vartheta} + x_2 \sin{\vartheta}
    \bigr)\cos{\vartheta}}{2}
    \Bigr] \\
    &= \eta [ 
    (1 - \cos^2\vartheta)(1-\Psi) - x_2 \sin{\vartheta}\cos{\vartheta}] \\  
    &=\eta [ (1-\Psi)\sin^2\vartheta -x_2\sin \vartheta \cos \vartheta ] 
\end{split}
\end{equation*}
We see that $\mathcal{I}$ only depends on $x_2$, and we can therefore integrate out the remaining integration variables $x_3,\ldots ,x_{d+1}$. Let $f(x_2)$ denote the marginal distribution of $x_2$. Then
\begin{equation*}
 \Delta (\Psi) \approx \int_{-\infty}^{\mathcal{B}(\Psi)}\eta [(1-\Psi)\sin^2\vartheta -x_2\sin \vartheta\cos \vartheta] f(x_2)dx_2
\end{equation*}
In the above derivation we assumed that the resulting integral is finite, and so is the integral in $\iR^{d+1}$ of the remaining terms in the Taylor expansion corresponding to $O(\eta^2)$.
\end{proof}

We can now state the main theorem of this section. 
\begin{theorem}
\label{thm:hingeglobal}
Assume that the gradient step size is small enough so that we can neglect second order terms $O(\eta^2)$. The expected convergence rate decreases monotonically as a function of $\Psi$ for every $\Psi > (1 - \cos{\vartheta})$ when $\cos\vartheta> 0$ ($\bar\bw,\bw_t$ are positively correlated), and for every $\Psi < (1 - \cos{\vartheta})$ when $\cos\vartheta< 0$. Monotonicity holds $\forall \Psi$ when $\cos\vartheta= 0$.
\end{theorem}
\begin{proof}

Using Lemma~\ref{lemma:integral} and the Leibniz Theorem for derivation under the integral sign, we get
\begin{equation*}
    \frac{\partial \Delta ( \Psi )}{\partial \Psi} = \Delta_1 + \Delta_2 
\end{equation*}
where
\begin{equation*}
\begin{split}
    \Delta_1 &= \eta [(1-\Psi)\sin^2\vartheta -x\sin \vartheta \cos \vartheta ] f(\mathcal{B}(\Psi)) \frac{\partial \mathcal{B}(\Psi)}{\partial \Psi} \\
    &\frac{\partial \mathcal{B}(\Psi)}{\partial \Psi}=\frac{\cos \vartheta}{\sin \vartheta} \\
    \Delta_2 &= \int\limits_{-\infty}^{\mathcal{B}(\Psi)}\frac{\partial}{\partial \Psi}\eta [(1-\Psi)\sin^2\vartheta -x_2\sin \vartheta \cos \vartheta] f(x_2) dx_2\\
    &= \int_{-\infty}^{\mathcal{B}(\Psi)}-\eta\sin^2\vartheta f(x_2) dx_2 \cr
\end{split}
\end{equation*}
Clearly $\Delta_2\le 0$. It therefore suffices to analyze the sufficient condition $\Delta_1\le0$ in order to conclude the proof. 

\textbf{Case 1:} $\cos\vartheta= 0$, where $\Delta_1= 0\implies \frac{\partial \Delta ( \Psi )}{\partial \Psi}<0~~\forall \Psi$.

\textbf{Case 2:} $\cos\vartheta> 0$. Since $f(x)>0$ (a density function), $\Delta_1\le0$ iff the first multiplicand in the expression describing $\Delta_1$ above is non-negative. By substituting $\mathcal{B}(\Psi)$ into this term we get:
\begin{equation*}
\begin{split}
   (1-&\Psi)\sin^2\vartheta -\mathcal{B}(\Psi)\sin \vartheta \cos \vartheta \\
    &= (1-\Psi)\sin^2\vartheta - 
    [ (\Psi - 1) \cos \vartheta + 1] \cos \vartheta \\
    &= 1 - \cos \vartheta - \Psi 
\end{split}
\end{equation*}
Clearly $\forall ~\Psi > (1 - \cos\vartheta)$ this term is negative. 

\textbf{Case 3:} $\cos\vartheta< 0$. Using the same line of reasoning, now $\Psi > (1 - \cos\vartheta)\implies \Delta_1<0$ since $\frac{\partial \mathcal{B}(\Psi)}{\partial \Psi}<0$. 
\end{proof}

Early in the training procedure we expect \textbf{Case 1}, when $\cos\vartheta > 0$ and $\bar\bw,\bw_t$ are positively correlated, to dominate SGD learning. This is because in a high dimensional space, two randomly picked vectors are expected to be almost orthogonal to each other, and therefore only a small step towards the optimal hypothesis is needed in order to satisfy this condition. Now the relevant condition is $\Psi > 1 - \cos\vartheta$, defining a range which includes almost all the training data with non-zero \emph{Difficulty Score}.

The condition on $\Psi$ in the theorem is necessary. To see this, we next show that when $\cos{\vartheta} > 0$ and $0<\Psi < 1 - \cos\vartheta$, there are cases for which the theorem does not hold. Similar construction exists when $\cos{\vartheta} < 0$ and $\Psi > 1 - \cos\vartheta$.

\begin{theorem}
\label{thm:low-psi-counter}
For all $w_t$ and when $\cos{\vartheta} > 0$, there exists $\mathcal{D}$ for which $\Delta(\Psi)$ is not monotonically decreasing with $\Psi$ in the range $[0, 1 - \cos{\vartheta}]$.
\end{theorem}
\begin{proof}
Let $0<\Psi_1<\Psi_2 < 1 - \cos\vartheta$. Assume that $f(x_2)=0 ~\forall x_2\le\mathcal{B}(\Psi_1)$, thus $\Delta (\Psi_1)=0$. It remains to show that $\Delta (\Psi_2)>0$. From Lemma~\ref{lemma:integral}
and neglecting second order terms in $\eta$
\begin{equation*}
\begin{split}
\Delta (\Psi_2) &\approx \eta \int_{\mathcal{B}(\Psi_1)}^{\mathcal{B}(\Psi_2)} \mathcal{J}(x_2)f(x_2) dx_2\\
\mathcal{J}(x_2)&=(1-\Psi_2)\sin^2\vartheta -x_2 \sin{\vartheta} \cos{\vartheta}
\end{split}
\end{equation*}
We next observe that $\mathcal{J}(x)>0$ $\forall x$ where $\mathcal{B}(\Psi_1)\le x\le\mathcal{B}(\Psi_2)$. This is because  $\mathcal{J}(x)$ is monotonically deceasing with $x$, and $\mathcal{I}(\mathcal{B}(\Psi_2))>0$ for $\Psi_2 < 1 - \cos\vartheta$. It thus follows that $\Delta (\Psi_2)>0$, which concludes the proof.
\end{proof}

\subsection{Convergence rate increases with \emph{local difficulty}}

In a similar manner to the case of linear regression and under the same assumptions, we show that when $\Psi$ is fixed, the convergence rate with respect to the local difficulty $\Upsilon$ is increasing, opposite to its trend with $\Psi$. 

As in Section~\ref{sec:regression-local}, we define:
\begin{equation*}
    \Delta (\Psi, \Upsilon) = \iE \left[ \frac{\bw_{t  + 1}\cdot \bar{\bw}}{\Vert \bw_{t  + 1}\Vert ~\Vert \bar{\bw} \Vert } - \frac{\bw_{t }\cdot \bar{\bw}}{\Vert \bw_{t }\Vert ~\Vert \bar{\bw} \Vert } \; \bigg{\vert} \; \Psi, \Upsilon \right]
\end{equation*}

\begin{theorem}
Assume that the gradient step size is small enough so that we can neglect second order terms $O(\eta^2)$. Assume further that $\cos{\vartheta}\geq  0$. Fixing $\Psi$ and $\forall\Psi$, the expected convergence rate is  monotonically increasing with $\Upsilon$ for every $\Upsilon > 0$.
\end{theorem}
\begin{proof}

From Fig.~\ref{fig:hinge-plain} we see when $\Psi,\Upsilon$ are given, the projection of data point $\bx$ onto $X_1\times X_2$ is a point where $x_1=1-\Psi$, and 
\begin{equation*}
\begin{split}
    (\cos\vartheta,&\sin\vartheta)\cdot(1-\Psi,x_2)=1-\Upsilon \\
    &\implies x_2=\mathcal{X}(\Psi,\Upsilon)=\frac{\Psi-1}{\tan\vartheta}+\frac{1-\Upsilon}{\sin\vartheta}
\end{split}
\end{equation*}

In the same manner used to prove Lemma~\ref{lemma:integral}, we can show that
\begin{equation*}
    \Delta (\Psi, \Upsilon) = \eta [
    (1-\Psi)\sin^2\vartheta -\mathcal{X}(\Psi,\Upsilon) \sin{\vartheta} \cos{\vartheta}] + O(\eta^2)
\end{equation*}
It follows that
\begin{equation*}
    \frac{\partial \Delta ( \Psi, \Upsilon)}{\partial \Upsilon} = \eta\cos{\vartheta}\geq  0
\end{equation*}
which concludes the proof.
\widowpenalty=10000
\end{proof}

\section{Summary and Discussion}

This paper offers the first theoretical investigation of curriculum learning, in the context of convex optimization. In its simplest form, curriculum learning can be viewed as a variation on stochastic gradient descent, where easy examples are more frequently sampled at the beginning of training. In our first contribution we defined how to measure example difficulty, using its loss with respect to the optimal hypothesis (global {Difficulty Score}). In our second contribution we analyzed two representative convex optimization problems - binary classification with hinge-loss minimization, and linear regression. For these two problems we showed that curriculum learning, with an initial bias in favor of training points whose loss with respect to the \emph{optimal hypothesis} is \textbf{lower}, accelerates learning. We also showed that when the global {Difficulty Score} is fixed, convergence is accelerated when preferring training points whose loss with respect to the \emph{current hypothesis} (local {Difficulty Score}) is \textbf{higher}. 

These theoretical results can direct us towards the development of new practical methods which will incorporate both global and local scores in order to balance between easy and hard examples. One simple approach for achieving this would be to control the pace of the curriculum schedule according to the local score. More sophisticated algorithms can combine biases based on both scores. Our results also suggest that the correlation between local and global scores can predict whether methods that favor currently easier examples, like SPL, or methods that favor currently hard examples, like hard example mining, will perform better in specific tasks. 

For example, when learning from noisy data, we expect to see high correlation between the local and global difficulty scores, and therefore preference towards examples with low local score will also bias towards examples with low global score. In such cases SPL, which gives preference to examples with lower local score, can improve convergence, based on our theoretical analysis. On the other hand, if the local and global scores are not correlated, hard data mining is likely to perform better based on our theoretical analysis.

\section*{Acknowledgements}

This work was supported in part by a grant from the Israeli Science Foundation (ISF) and by the Gatsby Charitable Foundations.
%\newpage

\bibliography{bib}

\begin{thebibliography}{15}
\providecommand{\natexlab}[1]{#1}
\providecommand{\url}[1]{\texttt{#1}}
\expandafter\ifx\csname urlstyle\endcsname\relax
  \providecommand{\doi}[1]{doi: #1}\else
  \providecommand{\doi}{doi: \begingroup \urlstyle{rm}\Url}\fi

\bibitem[Bengio et~al.(2009)Bengio, Louradour, Collobert, and
  Weston]{bengio2009curriculum}
Yoshua Bengio, J{\'e}r{\^o}me Louradour, Ronan Collobert, and Jason Weston.
\newblock Curriculum learning.
\newblock In \emph{Proceedings of the 26th annual international conference on
  machine learning}, pages 41--48. ACM, 2009.

\bibitem[Chang et~al.(2017)Chang, Learned-Miller, and
  McCallum]{chang2017active}
Haw-Shiuan Chang, Erik Learned-Miller, and Andrew McCallum.
\newblock Active bias: Training more accurate neural networks by emphasizing
  high variance samples.
\newblock In \emph{Advances in Neural Information Processing Systems}, pages
  1002--1012, 2017.

\bibitem[Elman(1993)]{elman1993learning}
Jeffrey~L Elman.
\newblock Learning and development in neural networks: The importance of
  starting small.
\newblock \emph{Cognition}, 48\penalty0 (1):\penalty0 71--99, 1993.

\bibitem[Jiang et~al.(2017)Jiang, Zhou, Leung, Li, and
  Fei-Fei]{jiang2017mentornet}
Lu~Jiang, Zhengyuan Zhou, Thomas Leung, Li-Jia Li, and Li~Fei-Fei.
\newblock Mentornet: Regularizing very deep neural networks on corrupted
  labels.
\newblock \emph{arXiv preprint arXiv:1712.05055}, 2017.

\bibitem[Kumar et~al.(2010)Kumar, Packer, and Koller]{kumar2010self}
M~Pawan Kumar, Benjamin Packer, and Daphne Koller.
\newblock Self-paced learning for latent variable models.
\newblock In \emph{Advances in Neural Information Processing Systems}, pages
  1189--1197, 2010.

\bibitem[Mitchell(1980)]{mitchell1980need}
Tom~M Mitchell.
\newblock \emph{The need for biases in learning generalizations}.
\newblock Department of Computer Science, Laboratory for Computer Science
  Research, Rutgers Univ. New Jersey, 1980.

\bibitem[Mitchell(2006)]{mitchell2006discipline}
Tom~Michael Mitchell.
\newblock \emph{The discipline of machine learning}, volume~9.
\newblock Carnegie Mellon University, School of Computer Science, Machine
  Learning Department, 2006.

\bibitem[Oh et~al.(2015)Oh, Guo, Lee, Lewis, and Singh]{oh2015action}
Junhyuk Oh, Xiaoxiao Guo, Honglak Lee, Richard~L Lewis, and Satinder Singh.
\newblock Action-conditional video prediction using deep networks in atari
  games.
\newblock In \emph{Advances in neural information processing systems}, pages
  2863--2871, 2015.

\bibitem[Sanger(1994)]{sanger1994neural}
Terence~D Sanger.
\newblock Neural network learning control of robot manipulators using gradually
  increasing task difficulty.
\newblock \emph{IEEE transactions on Robotics and Automation}, 10\penalty0
  (3):\penalty0 323--333, 1994.

\bibitem[Schapire et~al.(1998)Schapire, Freund, Bartlett, and Lee]{boosting98}
Robert~E. Schapire, Yoav Freund, Peter Bartlett, and Wee~Sun Lee.
\newblock Boosting the margin: A new explanation for the effectiveness of
  voting methods.
\newblock \emph{The Annals of Statistics}, 26\penalty0 (5):\penalty0
  1651--1686, 1998.
\newblock ISSN 00905364.
\newblock URL \url{http://www.jstor.org/stable/120016}.

\bibitem[Schroff et~al.(2015)Schroff, Kalenichenko, and
  Philbin]{schroff2015facenet}
Florian Schroff, Dmitry Kalenichenko, and James Philbin.
\newblock Facenet: A unified embedding for face recognition and clustering.
\newblock In \emph{Proceedings of the IEEE conference on computer vision and
  pattern recognition}, pages 815--823, 2015.

\bibitem[Shrivastava et~al.(2016)Shrivastava, Gupta, and
  Girshick]{shrivastava2016training}
Abhinav Shrivastava, Abhinav Gupta, and Ross Girshick.
\newblock Training region-based object detectors with online hard example
  mining.
\newblock In \emph{Proceedings of the IEEE Conference on Computer Vision and
  Pattern Recognition}, pages 761--769, 2016.

\bibitem[Skinner(1990)]{skinner1990behavior}
Burrhus~Frederic Skinner.
\newblock \emph{The behavior of organisms: An experimental analysis}.
\newblock BF Skinner Foundation, 1990.

\bibitem[Wang and Cottrell(2015)]{wang2015basic}
Panqu Wang and Garrison~W Cottrell.
\newblock Basic level categorization facilitates visual object recognition.
\newblock \emph{arXiv preprint arXiv:1511.04103}, 2015.

\bibitem[Weinshall et~al.(2018)Weinshall, Cohen, and
  Amir]{weinshall2018curriculum}
Daphna Weinshall, Gad Cohen, and Dan Amir.
\newblock Curriculum learning by transfer learning: Theory and experiments with
  deep networks.
\newblock In \emph{Proceedings of the 35th International Conference on Machine
  Learning}, pages 5238--5246, 2018.

\end{thebibliography}
\bibliographystyle{iclr2019}

\appendix

\section{Bayesian Formulation}
\label{app:bayesian}

The main result in Theorem~\ref{theorem:1} depends on the assumption that the two possible labels when the difficulty score is fixed $y(\bx\vert_{\displaystyle{\Psi}})$: $y_1(\bx)=\bx\cdot\bar\bw+{\Psi}$ and $y_2(\bx)=\bx\cdot\bar\bw-{\Psi}$, are equally likely: $f_{{\cal D}(\iX)}(\bx,y_i(\bx)) = \frac{1}{2}f(r,\vartheta,\Phi)$. We now describe how this assumption can be relaxed in a Bayesian framework. 

Let $\bv$ denote an additional (vector) hyper-parameter of the distribution  ${\cal D}(\iX)$, such that $f_{{\cal D}(\iX)} = f_\bv(r,\vartheta,\Phi,y)$. Let $q(\bv)$ denote the distribution of $\bv$. Assume that the conditional marginal distribution of the data over $v$ does not depend on the label, namely
\begin{equation}
\int f_v(r,\theta,\Phi,y\vert_{\displaystyle{\Psi}}) q(v) dv = C f(r,\vartheta,\Phi) ~~~\forall \Psi
\label{eq:bayesian}
\end{equation}
for some constant $C$. It follow that the marginal conditional distribution of the data satisfies the required condition
\begin{equation*}
f_{{\cal D}(\iX)}(\bx,y\vert_{\displaystyle{\Psi}}) =\int f_\bv(\bx,y\vert_{\displaystyle{\Psi}}) q(v) dv  =f(r,\vartheta,\Phi,y)\vert_{y\in\{y_1,y_2\}}=\frac{1}{2}f(r,\vartheta,\Phi)
\end{equation*}

Thus assumption (\ref{eq:bayesian}) suffices for Theorem~\ref{theorem:1} to hold true in a Bayesian framework, when taking the average over all hyper-parameter values.

\section{Normalization of parameters vector}
\label{app:general_normalization}

Throughout the analysis in Section~\ref{sec:hinge} we assumed the constraint $\Vert \bw \Vert = 1$, but the results also apply to any norm $A$ where $\Vert \bw \Vert = A$. To see this, let us define $\bx'=A\bx$. Define the following distribution $\mathcal{D'}$ on $\bX'$
\begin{equation*}
    \forall \bx',\; y: \; \mathcal{D'}(\bx',y) = \mathcal{D} (A\bx, y) 
\end{equation*}
We note that
\begin{equation*}
\begin{split}
\mathop{\argmin}_{\bw,~s.t.\Vert\bw\Vert=A}\max(1-(\bx\cdot\bw) y , 0) = &\mathop{\argmin}_{\bw,~s.t.\Vert\bw\Vert=1}\max(1-(A\bx\cdot\bw) y , 0) \\
=&\mathop{\argmin}_{\bw,~s.t.\Vert\bw\Vert=1}\max(1-\bx'\cdot\bw) y , 0)
\end{split}
\end{equation*}
The latter is the problem we have analyzed for any distribution on the training data, including $\mathcal{D'}$. Thus the theorems we have proved hold true for this problem as well.

%We note that the difficulty score in the new problem $\Psi'$ satisfies $1-\Psi'=A(1-\Psi)$, since these expressions are the first coordinate of $\bx'$ and $\bx$ respectively. Therefore the constraint in Theorem~\ref{thm:hingeglobal} $\Psi > 1 - \cos \vartheta$ should be replaced with $\Psi > \max (0, 1 - \frac{\cos \vartheta}{A})$. 

\end{document}